%% file: functor_homology.tex
\pdfoutput=1
\documentclass{article}
\usepackage[a4paper]{geometry}
\usepackage{etex}
\usepackage[utf8]{inputenc}

\usepackage{amsmath}
\usepackage{amsthm}
\usepackage{amssymb}
\usepackage{centernot}
\usepackage{mathtools}
\usepackage{ stmaryrd }

\usepackage{dsfont}
\usepackage{bbm}
\usepackage{framed,color}
\definecolor{shadecolor}{rgb}{0.92,0.92,0.92}
\usepackage{textcomp}
\usepackage{upgreek}
\usepackage{wasysym}
\usepackage{rotating}
\usepackage{framed}
\usepackage{subfig}

\usepackage{xcolor}
\usepackage[linesnumbered,ruled,vlined]{algorithm2e}

\usepackage{multicol} 
\usepackage{animate}
\usepackage{colortbl}
\usepackage{tikz}
\usetikzlibrary{patterns}
\usetikzlibrary{matrix}

\newcommand{\R}{\mathbb{R}}

\renewcommand{\tilde}{\widetilde}

\usepackage{graphicx,here,comment}
\usepackage{ mathrsfs }
\usepackage{tikz}

\usepackage[hidelinks]{hyperref}

\usepackage[all,2cell]{xy}
\UseAllTwocells
\usepackage{pb-diagram}
\usepackage{pb-xy}

\usepackage{tikz-cd}
\usepackage{bm}

\usepackage[english]{babel}
\newtheorem{theorem}{Theorem}[section]
\newtheorem{lemma}[theorem]{Lemma}
\newtheorem{proposition}[theorem]{Proposition}

\theoremstyle{definition}
\newtheorem{definition}[theorem]{Definition}

\newtheorem{example}[theorem]{Example}
\newtheorem{remark}[theorem]{Remark}

\newcommand{\Ima}{\operatorname{Im}}
\newcommand{\Ker}{\operatorname{Ker}}

\newcommand{\lMod}[1]{{#1}\text{-}\mathbf{Mod}} 
\newcommand{\rMod}[1]{\mathbf{Mod}\text{-}{#1}} 

\newcommand{\op}{\mathrm{op}}

\newcommand{\Z}{\mathbb{Z}}
\newcommand{\barN}{\overline{N}}
\newcommand{\barC}{\overline{C}}
\newcommand{\rank}{\operatorname{rank}}

\newcommand{\Path}{\operatorname{Path}}

  \newcommand{\rarrow}[1]{\buildrel #1 \over \longrightarrow}
  \newcommand{\larrow}[1]{\buildrel #1 \over \longleftarrow}

\newcommand{\coequalizer}[5]{\xymatrix{
 {\displaystyle #1} 
 \ar@<1ex>[r]^-{#2} \ar@<-1ex>[r]_-{#3}
 & {\displaystyle #4} \ar[r] & {\displaystyle #5}
 }%
}
\newcommand{\equalizer}[5]{
\xymatrix{
 {\displaystyle #1} \ar[r] & {\displaystyle #2}
 \ar@<1ex>[r]^-{#3} \ar@<-1ex>[r]_-{#4}
 & {\displaystyle #5} 
 }%
}

\title{A Weighted Quiver Kernel using Functor Homology}
\author{Manohar Kaul and Dai Tamaki}

\begin{document}

\maketitle

\begin{abstract}
 In this paper, we propose a new homological method to study weighted
 directed networks. 
 Our model of such networks is a directed graph $Q$ equipped with a
 weight function $w$ on the set $Q_{1}$ of arrows in $Q$.
 We require that the range $W$ of our weight function is equipped with
 an addition or a multiplication, i.e.~$W$ is a monoid in the
 mathematical terminology.
 When $W$ is equipped with a representation on a vector space $M$, 
 the standard method of homological algebra allows us to define the
 homology groups $H_{*}(Q,w;M)$.

 It is known that when $Q$ has no oriented cycles,
 $H_{n}(Q,w;M)=0$ for $n\ge 2$ and $H_{1}(Q,w;M)$ can be easily computed.
 This fact allows us to define a new graph kernel for weighted directed
 graphs. We made two sample computations with real data and found that
 our method is practically applicable.
\end{abstract}


\input{introduction}

\input{weighted_category}

\input{feature_vector}

\input{applications}

\appendix

\input{math}

\bibliographystyle{abbrv}
\bibliography{functor_homology} 

\end{document}

%% file: introduction.tex
\section{Introduction}
\label{introduction}

Graphs and quivers (directed graphs)\footnote{The rest of the paper uses the terms \emph{directed graph} and \emph{quiver} interchangeably.} are ubiquitous in mathematical
sciences.
In many applications, vertices or edges of graphs and
quivers are labeled and have \emph{costs} associated with them, also
called \emph{weights}.  
In this paper, we are interested in edge-weighted quivers. 
These weights are not restricted to 
just scalar values, but can also represent much more complex and richer
relations between the nodes of an edge by modeling them as \emph{label
sets} or a \emph{function of several variables}. 

Such weighted quivers arise frequently when modeling real-world
applications, especially where the \emph{relationships} among objects
play an important role. Below are a few applications of weighted quivers
that cover wide and diverse fields: 
\begin{itemize}
 \item \textbf{Physics}: weighted quivers are used to represent atomic
       structures, where an atom is depicted as a vertex and the
       interactive forces between the atoms (i.e., vertices) are shown
       as directed edges between pairs of vertices. The edge weights
       here can model the \emph{strength of interaction} between two
       vertices. Note that such a weighted quiver also accepts
       \emph{multiple edges} between the same pair of vertices, where
       each edge potentially represents a different type of interactive
       force. 
	
 \item \textbf{Chemistry:} weighted quivers model molecular structures,
       where the vertices and the edges represent atoms and the chemical
       bonds between them, respectively. The edge weights contain
       information such as the bond angles, the magnitude of
       electrostatic force of attraction, polarity of the bonds, etc. 
	
 \item \textbf{Neuroscience:} weighted quivers can represent a
       functional model of the brain, where vertices represent regions
       of the brain and the edges represent the connections or
       communication pathways between them. The edge weights can
       represent similarity between two brain signals at the vertices,
       information propagated between the vertices via the edge, etc.  
	
 \item \textbf{World Wide Web (WWW):} weighted quivers represent the
       interconnections between documents on the web, where web
       documents are shown as vertices and edges represent the
       references between them. An edge weight in this instance could
       signify the number of times the source vertex referenced the
       target vertex, or how many web links they share in common etc. 	
\end{itemize}

We focus our attention to implementing a kernel method in the
study of such weighted networks.
Recall that, given a family of graphs $\mathcal{G}$, a
\emph{graph kernel} on $\mathcal{G}$ is a function
$k:\mathcal{G}\times\mathcal{G}\to\R$ defined by
\[
 k(G,G') = \langle \phi(G),\phi(G')\rangle
\]
for $G,G'\in\mathcal{G}$, where $\phi : \mathcal{G} \to \R^{N}$ is an
embedding, called a \emph{feature map} and $\langle-,-\rangle$ is the
standard inner product in $\R^{N}$.
The kernel method was introduced in the field
of machine learning 
\cite{Gaertner-Flach-Wrobel2003,Kashima-Tsuda-Inokuchi2003}.
Since then quite a few graph kernels have been proposed for graphs and
labelled graphs. 
Such graph kernels are proposed to answer two
often-encountered questions, in the context of graphs. Namely, ``How
similar are two nodes in a given graph?'' and  
``How similar are two graphs to each other?''. 
More details on graph kernels can be
obtained from the survey paper \cite{1903.11835}. 

The novelty of our method is the use of a homology theory for weighted
quivers in the construction of a feature map.
Given a quiver $Q$, a weight function $w:Q_{1}\to W$ and a
representation (action) of $W$ on a vector space $M$, we define homology
groups $H_{*}(Q,w;M)$, called the \emph{weighted quiver homology}. 
Although the dimension theory of small categories (e.g.~\S1.6 of
\cite{Husainov2002}) implies 
that $H_{n}(Q,w;M)=0$ for $n\ge 2$, the first homology $H_{1}(Q,w;M)$
contains essential information of the weighted quiver $(Q,w)$.
Furthermore we have an explicit description of $H_{1}(Q,w;M)$, giving us
a computable invariant. See Theorem \ref{first_homology} for a precise
statement.  

In order to construct a feature map, we order the vertex set
$Q_{0}=\{v_{1},\ldots,v_{N}\}$ and choose a positive integer $H$.
For each vertex $v_{i}$, we iterate $H$
times, each time computing a progressively larger acyclic sub-quiver and
the dimension of its first weighted quiver homology, denoted by
$h_{k}(v_{i})$ in the $k$-th iteration.
These numbers form a vector 
$\bm{h}(v_{i})=(h_{1}(v_{i}),h_{2}(v_{i}),\ldots,h_{H}(v_{i}))\in
\R^{H}$. 
The sequence $(\bm{h}(v_{1}),\bm{h}(v_{2}),\ldots,\bm{h}(v_{N}))\in
\underbrace{\R^{H}\times\cdots\times\R^{H}}_{N} = \R^{HN}$ is our
feature vector.

We remark that this approach is inspired by the neighborhood
aggregation approaches outlined in graph kernel literature in the area
of machine learning, especially the \emph{Weisfeiler-Lehman} (WL)
kernel~\cite{Shervashidze-Schweitzer-vanLeeuwen-Mehlhorn-Borgwardt2011}.
An overarching principle in the design of 
graph kernels is the representation and comparison of local structure in
graphs. Two vertices are considered \emph{similar} if their
neighborhoods are colored / labeled similarly. A natural extension to
this notion is that two graphs are considered similar if they are
composed of vertices with similar neighborhoods, i.e., they have a
similar local structure. 

In neighborhood aggregation schemes, each vertex in a graph is assigned
a color or attribute based on a summary of the local structure
surrounding the vertex. For each vertex, iteratively, the attributes /
colors are aggregated to compute a new attribute / color that eventually
represents the structure of its extended neighborhood in a compressed
and compact form.
Shervashidze et
al.~\cite{Shervashidze-Schweitzer-vanLeeuwen-Mehlhorn-Borgwardt2011} 
introduced a highly influential class of neighborhood aggregation
kernels for graphs with discrete labels based on the 1-dimensional
Weisfeiler-Lehman (1-WL) or \emph{color refinement
algorithm}~\cite{Babai-Kucera1979}: a well-known heuristic for the
\emph{graph isomorphism problem}.
Our approach can be thought of as an implementation of the WL kernel for
weighted networks by using the weighted quiver homology. 




We made two sample computations of our feature vectors on the following
examples. 

\begin{example}[Node Embeddings of Weighted Directed Graphs
 (Section~\ref{subsec:node})] 
 Machine learning (ML) methods favor \emph{continuous vector}
 representations, while graphs are inherently unordered, irregular, and
 combinatorial in nature.  
 A popular task in ML is to find \emph{graph embeddings} to represent a
 graph such that the embedding captures the graph's original shape,
 linkage structure, and other graph properties (e.g. cliques, cycles
 etc.). The more graph properties a graph embedding captures the better
 are the downstream tasks like classification of graphs, or predicting
 future link creation etc.  
 Roughly, there are two types of embeddings:
 \begin{enumerate}
  \item vertex/node embeddings where two vertices in a graph surrounded
	by similar local structures are also found close to one another
	in the vertex embeddings, and 
  \item graph embeddings where two graphs with similar properties
	cluster together and two graphs with dissimilar properties
	appear farther from each other in this vector space.
 \end{enumerate}
 We refer the reader to a survey on node embeddings~\cite{1808.02590} for
 more details.

 We computed the feature vectors of the \emph{Cora dataset}
 \cite{Sen-Namata-Bilgic-Getoor-Gallagher-EliassiRad2008}, which is a
 research citation network (directed) comprising of $2708$ scientific
 publications classified into one of \emph{seven categories}.
 In this experiment, nodes that represent a given topic cluster together
 and also move away from topics that are different. We see this
 separation improve as we vary the number of iterations $H$ from $4$ to $6$. 
 See Figure~\ref{fig:node_emb}.
\end{example}

\begin{example}[Community Detection in Weighted Graphs
 (Section~\ref{subsec:comm})] 
 One of the most relevant features of graphs representing real systems
 is \emph{community structure}, or \emph{clustering}, i.e., the
 organization of vertices in clusters, with many edges joining 
vertices of the same cluster and comparatively few edges joining
 vertices of different clusters.  
 Such communities can be considered as independent components of a
 graph, that play a very similar role, e.g., the tissues or the organs
 in the human body. Community detection finds applications in a wide and
 diverse set of areas such as biology, sociology, and computer science,
 to name a few, where systems are often represented as graphs. This
 problem is extremely hard and has not yet been solved satisfactorily,
 despite the huge effort of a large interdisciplinary community of
 scientists working on it over the past few years. This task gets even
 harder when having to identify such communities in weighted directed
 graphs. We refer the reader to a survey on community
 detection \cite{Fortunato2010} for more details. 

 For our experiment, we used the Facebook graph dataset from SNAP
 \cite{McAuley-Leskovec2012}.
 It can be visually observed from  Figure~\ref{fig:comm}(b) that
 our method does a fairly good job of detecting communities
 in the strong sense in the Facebook graph. 
\end{example}


The paper is organized as follows.
\begin{itemize}
 \item \S\ref{weighted_category} is preliminary. We collect notation and
       terminology used in this paper.
 \item Our feature map is defined in \S\ref{feature_vector}. After
       recalling the idea of the homology of small categories in
       \S\ref{small_category}, the weighted quiver homology is defined
       in \S\ref{homology_of_weighted_category}.
       The algorithm for computing the feature map is described in
       \S\ref{weighted_quiver_kernel}. 

 \item Applications to two practical examples are described in
       \S\ref{applications}. 

 \item An appendix is attached in which mathematical details 
       lying behind our weighted quiver homology are described.
\end{itemize}


%% file: weighted_category.tex
\section{Weighted Quivers and Weighted Categories}
\label{weighted_category}

This section is preliminary.
Here we summarize notation and terminology for weighted directed graphs
and related structures used in this paper.

\input{quiver}
\input{weight_function}

%% file: quiver.tex
\subsection{Graphs, Quivers, and Small Categories}
\label{quiver}

A graph whose edges are directed is often called a \emph{directed graph}
or a \emph{digraph}, for short, in applied mathematics, where 
digraphs are often assumed to be simple, i.e.~there are at most one edge
between two vertices.
On the other hand, directed graphs are also used in pure
mathematics, such as representation theory, in which they are usually
called \emph{quivers} and are not assumed to be simple.
In this paper, we use the term quiver.

\begin{definition}
 \label{def:quiver}
 A \emph{quiver} $Q$ consists of two sets $Q_0$, the set of
 \emph{vertices}, and $Q_1$, the set of \emph{arrows}.
 When an arrow $u\in Q_{1}$ is directed from a vertex $x$ to another
 vertex $y$, we write $u: x \to y$. The vertices are also written as
 $s(u) = x$ and $t(u) = y$ so that we obtain the \emph{source} and the
 \emph{target} maps
 \[
  s,t : Q_{1} \rarrow{} Q_{0}.
 \]
 The set of arrows from $x$ to $y$,
 i.e.~$s^{-1}(x) \cap t^{-1}(y)$, is denoted by $Q(x, y)$.  

 A quiver $Q$ is called \emph{simple} if, there is at most one arrow
 between each pair of distinct vertices and there is no arrow of the
 form $x\to x$.
\end{definition}

\begin{remark}
 \label{remark:arrows_in_simple_quiver}
 When $Q$ is simple, the map
 \[
  s\times t : Q_{1} \rarrow{} Q_{0}\times Q_{0}
 \]
 is injective and the set of arrows $Q_{1}$ can be regarded as a subset
 of $Q_{0}\times Q_{0}$.
 In particular, an arrow $u:x\to y$ in $Q$ is represented by the pair of
 vertices $(x,y)$.
\end{remark}

\begin{remark}
 The sets of vertices and arrows of a quiver $Q$ are
 sometimes denoted by $V(Q)$ and $E(Q)$, respectively. When we consider
 generalizations to hypergraphs, however, our notation will be more
 convenient. 
\end{remark}

The notion of paths is essential in the study of quivers.

\begin{definition}
 \label{def:path}
 By a \emph{path} $\gamma$ on a quiver $Q$, we mean a finite sequence of
 composable arrows in $Q$, i.e.~$\gamma=(u_{n},u_{n-1},\ldots,u_{1})$
 such that $t(u_{i})=s(u_{i+1})$ for all $i=1,\ldots,n-1$. The number
 $n$ is called the \emph{length} of $\gamma$.
 The set of paths of length $n$ in $Q$ is denoted by $N_{n}(Q)$. By
 convention, $N_{0}(Q)=Q_{0}$.

 The obvious extensions of the source and the target maps are denoted by
 \[
  s,t : N_{n}(Q) \rarrow{} Q_{0},
 \]
 respectively.
\end{definition}

The observation in Remark \ref{remark:arrows_in_simple_quiver} can be
extended as follows.

\begin{remark}
 \label{remark:paths_in_simple_quiver}
 Let $x_{i}=t(u_{i})=s(u_{i+1})$ in a path
 $\gamma=(u_{n},\ldots,u_{1})$. Then $\gamma$ can be expressed as 
 \[
  x_{0} \rarrow{u_{1}} x_{1} \rarrow{u_2} \cdots \rarrow{u_{n}} x_{n}.
 \]
 Note the reversal of the ordering of arrows.
 When $Q$ is simple, this path can be represented by the sequence of
 vertices $(x_{0},\ldots,x_{n})$.
\end{remark}

By regarding paths as arrows, we obtain new quivers.

\begin{definition}
 \label{def:power_of_quiver}
 For a quiver $Q$, define a quiver $\Path(Q)$ as follows. The set of
 vertices is the same as that of $Q$; $\Path(Q)_{0}=Q_0$. Arrows in
 $\Path(Q)$ are paths in $Q$;
 \[
  \Path(Q)_{1} = \coprod_{n=1}^{\infty} N_{n}(Q).
 \]
 The source and target maps are defined in Definition
 \ref{def:path}. 
 This is called the \emph{path quiver of $Q$}.
%
\end{definition}

The quiver $\Path(Q)$ contains $Q$ as a subquiver. An important
difference is that we may compose arrows in $\Path(Q)$. This composition
operation makes $\Path(Q)$ very close to being a small category.

A \emph{small category} is a category whose objects form a set. In other
words, it consists of the set of objects $C_{0}$, the set of morphisms
$C_{1}$, and the composition law of morphisms. It is also required that
the identity morphism $1_{x}$ is assigned to each object $x\in C_{0}$.
A precise description is given as follows.

\begin{definition}
 \label{def:small_category}
 A \emph{small category} $C$ consists of the following data:
 \begin{itemize}
  \item a quiver $(C_{0},C_{1},s,t)$,
  \item an operation, called the \emph{composition}, which assigns
	an arrow $u_{2}\circ u_{1}$ to each composable pair of arrows
	$(u_{2},u_{1})\in N_{2}(C)$, and 
  \item an assignment of a distinguished arrow $1_{x}:x\to x$, called
	the \emph{identity at $x$}, to each element $x\in C_{0}$.
 \end{itemize}
 They are required to satisfy the following conditions:
 \begin{enumerate}
  \item The composition is associative;
	$(u_{3}\circ u_{2})\circ u_{1}=u_{3}\circ(u_{2}\circ u_{1})$ for
	each composable triple $(u_{3},u_{2},u_{1})\in N_{3}(C)$.
  \item When $u:x\to y$, $1_{y}\circ u=u=u\circ 1_{x}$.
 \end{enumerate} 
\end{definition}

\begin{remark}
 When $C$ is a small category, elements of $C_{0}$ and $C_{1}$ are
 called \emph{objects} and \emph{morphisms}, respectively. Elements of
 $N_{n}(C)$ are called \emph{$n$-chains} or \emph{chains of length $n$},
 instead of paths. 
\end{remark}

By adding identity morphisms to the path quiver $\Path(Q)$, we obtain a
small category.

\begin{definition}
 For a quiver $Q$, the small category obtained by adding
 $N_{0}(Q)=Q_{0}$ to $\Path(Q)$ as identity morphisms is denoted by
 $F(Q)$. Thus
 \[
 F(Q)_{1} = \coprod_{n=0}^{\infty} N_{n}(Q).
 \]
 This is called the \emph{free category generated by $Q$}.
 It is also called the \emph{path category of $Q$}.
 The composition is given by the concatenation of paths.
\end{definition}

%% file: weight_function.tex
\subsection{Weight Functions on Quivers and Small Categories}
\label{weight_function}

In practical applications, graphs and quivers often have labels on
their vertices or arrows. Coloring vertices is one of central topics in
graph theory.
In this paper, we are interested in colorings of arrows. 
The following general definition is borrowed from a paper
\cite{1802.03546} by Kanda, in which the term \emph{color} is used
instead of weight.

\begin{definition}
 An \emph{arrow-weight}, or simply a \emph{weight}, of a quiver
 $Q$ with weights in a set $W$ is a map $w:Q_{1}\to W$.
 A \emph{weighted quiver} is a pair $\Gamma=(Q,w)$ of a quiver $Q$
 and its arrow-weight $w:Q_{1}\to W$.
\end{definition}

In order to introduce compositions of arrows in a weighted quiver, we 
need an amalgamation of weights. Such an operation should be
associative. In other words, $W$ should be a semigroup.

\begin{lemma}
 \label{canonical_extension}
 If the set of weights $W$ of a weighted quiver $(Q,w)$ has a structure of
 semigroup, the wegith $w$ has a canonical extension
 \[
  \tilde{w} : \Path(Q)_{1} \rarrow{} W
 \]
 given by
 \[
  \tilde{w}(u_{n},\ldots,u_{1}) = w(u_{n})\cdot w(u_{n-1}) \cdots w(u_{1}),
 \]
 where the multiplication in $W$ is denoted by $\cdot$.
 When $W$ is a monoid with unit $1$, it can be further extended to
 \[
  \tilde{w}: F(Q)_{1} \rarrow{} W
 \]
 by $\tilde{w}(x)=1$ for $x\in N_{0}(Q)=Q_{0}$.
\end{lemma}

Note that the new weight function $\tilde{w}$ transforms compositions of
paths into multiplications (amalgamations) of weights;
\[
 \tilde{w}(\gamma\circ\delta) = \tilde{w}(\gamma)\tilde{w}(\delta).
\]
We require this property for weights of small categories.

\begin{definition}
 \label{def:weighted_small_category}
 A \emph{weight function} on a small category $C$ with weights in a
 monoid $M$ is a weight $w:C_{1}\to M$ such that
 \begin{enumerate}
  \item the weight function $w$ preserves units in the sense that
	$w(1_{x})=1$ for any object $x$, and
  \item the weight function $w$ is multiplicative in the sense that
	\[
	w(u\circ v) = w(u)w(v)
	\]
	for any composable pair $(u,v)$ of morphisms in $C$.
 \end{enumerate}

 The pair $(C,w)$ of a small category $C$ and a weight function $w$ is
 called a \emph{weighted small category}.
\end{definition}

\begin{example}
 \label{example:from_weighted_quiver_to_weighted_category}
 For a weighted quiver $(Q,w)$, the pair $(F(Q),\tilde{w})$ is a
 weighted small category.
\end{example}

By Lemma \ref{canonical_extension}, the power construction in Definition 
\ref{def:power_of_quiver} can be extended to weighted quivers. The weight 
function of the $\ell$-th power of a weighted quiver $\Gamma=(Q,w)$ is
denoted by
\[
 w : \Path_{\ell}(\Gamma)_{1} = \Path_{\ell}(Q)_{1} \rarrow{} W.  
\]

\begin{example}
 Figure~\ref{fig:qpower} shows an example of a weighted quiver $\Gamma$ 
 and its \emph{$2$nd power} $F^{2}(\Gamma)$.
 \begin{figure}[htbp]	
  \centering
  \includegraphics[width=.5\textwidth]{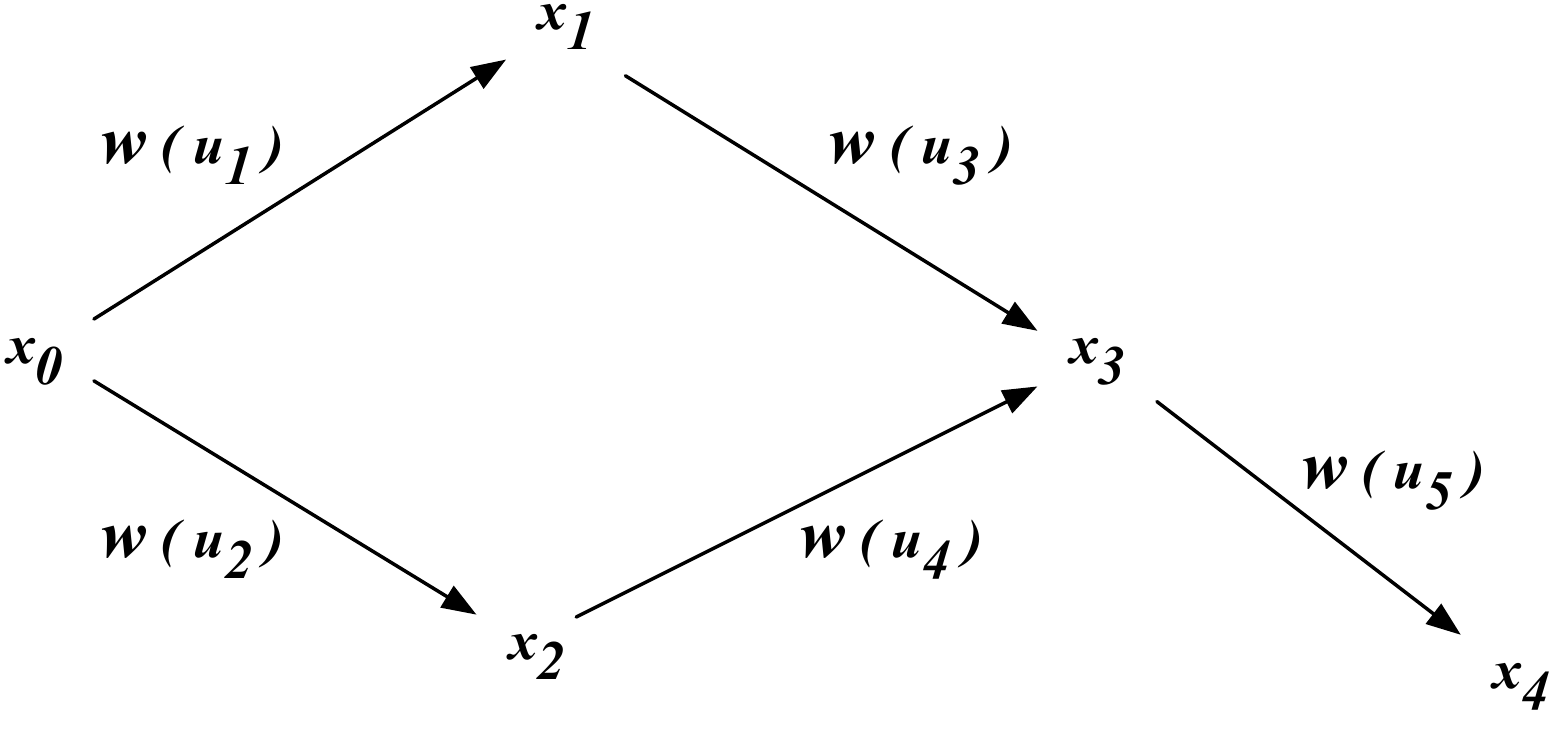}\hfill
  \includegraphics[width=.5\textwidth]{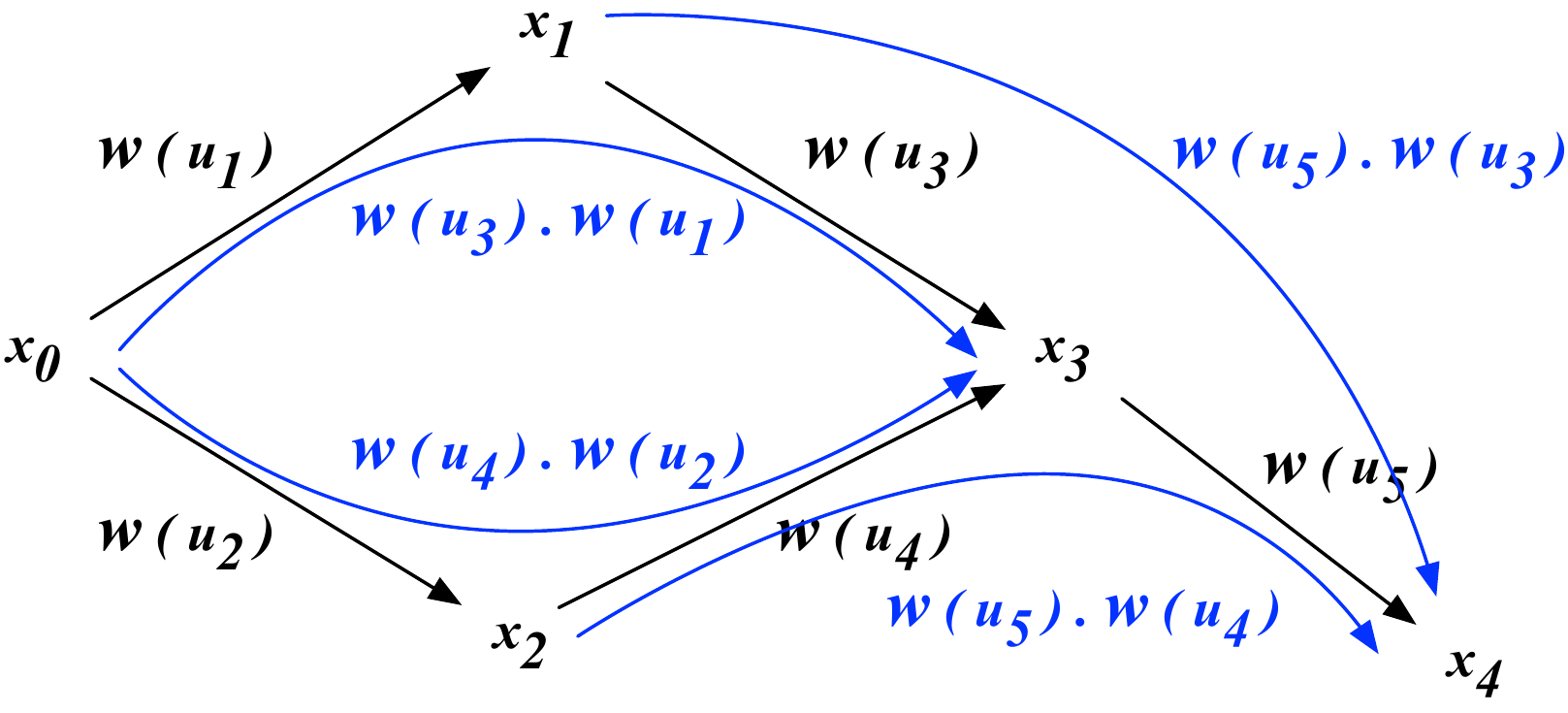}\hfill
  \caption{(left-to-right): Weighted quivers (a) $\Gamma$ and (b)
  $\Path_{2}(\Gamma)$.} 
  \label{fig:qpower}
 \end{figure} 
\end{example}

%% file: feature_vector.tex
\section{A Feature Map using Weighted Quiver Homology}
\label{feature_vector}

In this section, we define a homology theory for weighted quivers, with
which a new ``weighted quiver kernel'' is defined.
Throughout this section, we fix a commutative ring $k$. When necessary,
we assume that $k$ is a field.
 
\input{small_category}

\input{homology_of_weighted_category}

\subsection{A Weighted Quiver Kernel by Homology}
\label{weighted_quiver_kernel}


After having introduced the necessary terminology and developed a
weighted quiver homology, we are now ready to describe our method for
constructing feature vectors.  

Before we describe our algorithm, we explain how one constructs an
\emph{acyclic quiver} (or \emph{directed acyclic graph (DAG)}) from a
simple quiver (or directed graph). In order to break cycles and leave a
quiver ``acyclic'', one must 
identify and remove a minimum set of arrows. In graph theory, this is a
well-known NP-hard problem, referred to as the \emph{minimum feedback
arc set problem}.  
Due to the NP-hard nature of this problem, we resort to a randomized
approximation algorithm proposed by Berger and Shor
\cite{Berger-Shor1990}.

\begin{algorithm}[tbp]
	\DontPrintSemicolon 
	\KwIn{
		Simple weighted quiver $\Gamma=(Q,w)$
	}
	\KwOut{
		A feedback arc set $F$ for $\Gamma$
	}
	\tcc{Initialization of Feedback arc set with empty set}
	$F \gets \emptyset$\;
	
	\tcc{Process vertices in a fixed permuted order}
	\For{ $v \in \Pi(Q_0)$ }{
		\tcc{if there are more incoming arrows than outgoing ones...}
		\If{ $\delta_{in}(v) > \delta_{out}(v) $ }{
			$E_1 := \{ e \mid e \in Q_1 , s(e)=v  \}$ \;
			$F \gets F \cup E_1$ \;
		}
		\Else{
			$E_2 := \{ e \mid e \in Q_1 , t(e)=v  \}$ \;
			$F \gets F \cup E_2$ \;
		}
		\tcc{Discard edges $E_1 \cup E_2$ from $Q_1$}
		$ Q_1 \gets Q_1 \setminus (E_1 \cup E_2)$
	}
	\tcc{Set of feedback arcs dropped from $\Gamma$}
	\Return{$F$}\;
	\caption{{\sc Berger and Shor Algorithm to compute Feedback Arc Set.}}
	\label{algo:berger}
\end{algorithm}

For the sake of completeness, Algorithm~\ref{algo:berger} describes the
Berger and Shor algorithm in detail. This algorithm begins by choosing a
random permutation $\Pi(Q_0)$ of the vertices of the incoming 
quiver $\Gamma$. The vertices are processed in the order given by the
permutation (Line $2$). 
If a given vertex $v$ has more incoming arrows than outgoing ones, 
then $E_1$ contains the outgoing nodes and this is added to our feedback
set $F$ (Lines $3$--$5$). 
The opposite case is handled on Lines $7$--$8$. The edges in $E_1 \cup
E_2$ are removed from $Q_1$ and the remaining arrows make $\Gamma$
acyclic. The set $F$ contains the feedback arcs/arrows that are
dropped. 

The intuition behind this approach is that we choose to keep either the
incoming or outgoing arrows at any given time which ensures that the
resulting quiver is acyclic. Additionally, we choose to keep the set of
incoming or outgoing arrows with larger cardinality, thus resulting in a
larger acyclic quiver. 
This randomized algorithm runs in $O(M+N)$ (where $M$ and $N$ denote the
number of arrows and vertices in the quiver) and produces an acyclic
quiver containing at least $1/2 + \Omega(1/ \sqrt{ \delta_{max} })|Q_0|$
arrows, where $\delta_{max}$ is the maximum degree of any vertex in
$Q_0$. 

\begin{algorithm}[tbp]
	\DontPrintSemicolon 
	\KwIn{
		Simple weighted quiver $\Gamma=(Q,w)$, Number of iterations $H$, Total number of nodes $N$ and arrows $M$ in $Q$
	}
	\KwOut{
		Feature matrix $X \in \mathds{Z}^{N \times H}$ representing $\Gamma$
	}
	\tcc{Initialization}
	$ X \gets $ Empty $N \times H$ matrix \;
	$i \gets 0$ \;
	
	\For{ $v \in Q_0$ }{
		$i \gets i+1$ \;
		
		\For{ $k \in \{ 1 \dots H\}$  }{
			$N_k(v) \gets $ Set of vertices in $Q$ within $k$-hops from $v$ \;
			\tcc{Makes use of Algorithm~\ref{algo:berger} to build DAG}
			$Q'_k \gets$ Sub-quiver DAG induced by vertices in $N_k(v)$ \; 
			$(n,m) \gets$ (\# of nodes, \# of edges) in $Q'_k$ \; 
			\tcc{Build the boundary matrix $D$ for subquiver $Q'_k$}
			$D \gets$ Empty $n \times m$ matrix \;
			\ForEach{ arrow $u \rarrow{w'} v$ in $Q'_k$ }{
				$D[u,i] \gets -1$ \;
				$D[v,i] \gets w'$ \;
			}
			$rank(D) \gets$ Compute rank of matrix $D$ \;
			\tcc{ Get $dim H_1(Q'_k,w'; k(w') $}
			$X[i,k] \gets m - rank(D)$
		}
	}
	\Return{$X$}\;
	\caption{{\sc Computes feature vectors based on weighted quiver homology.}}
	\label{algo:quiv}
\end{algorithm}

Algorithm~\ref{algo:quiv} describes in detail all the steps required for
feature computation. 
A high-level description of our algorithm consists of the following operations. 

For every vertex $v$ in the underlying quiver's vertex set $Q_0$ (line
$3$), we iterate $H$ times, each time computing a progressively larger
acyclic sub-quiver (in the form of a \emph{directed acyclic graph}
(DAG)) and its weighted quiver homology (lines $6$--$14$).  
Note that the variable $k$ ranges from $1$ to $H$, and in each iteration
for a given value of $k$, we compute the set of vertices $N_k(v)$ that
are $k$-hops away from $v$, i.e., the set of vertices with a directed
path of length at most $k$ from $v$. Finally, the dimensions of the
first homology for each $k$ are concatenated to form a vector of size
$H$ (line $14$). 
For a given simple weighted quiver $\Gamma = (Q,w)$ with $N$ nodes and
$M$ arrows, our procedure results in $N$ feature vectors, each of size
$H$. 

\subsection*{Time complexity}

The dominant costs in our computation are
incurred by the matrix rank computation and computing the
$k$-hop neighborhood.  

To begin with, we analyze the rank computation cost.
In the worst case, the dimension of matrix $D$ representing
a sub-quiver is $N \times M$, when the sub-quiver is the
same as the quiver $Q$.  
According to Golub and Van Loan \cite{Golub-VanLoan2013} the
best known rank computation algorithms that internally
involve \emph{singular value decomposition} (SVD) for a
$N \times M$ matrix has a time complexity of $O(NM^2)$. 

Next, we study the cost of computing the $k$-hop neighborhood.
Let $N^{(k)}_{\max}$ and $M^{(k)}_{\max}$ denote the maximum
number of vertices and edges, respectively, in a subquiver
induced by a $k$-hop neighborhood around a vertex. Then,
steps in lines $6$--$7$ have a time-complexity of
$O(N^{(k)}_{\max} + M^{(k)}_{\max}  )$.  

Then, lines $6$--$14$, have a total complexity of
$O(N^{(k)}_{\max} + M^{(k)}_{\max} + NM^2 )$. 
As this is repeated for each vertex (i.e., $N$ of them) and
for $H$ times, we get an overall time complexity of
$O(NH(N^{(k)}_{\max} + M^{(k)}_{\max} + NM^2  )  )$.

%% file: small_category.tex
\subsection{Homology of Small Categories}
\label{small_category}

Let us first recall the definition of homology of small categories.
The definition can be regarded as a variant of the homology of a
simplicial complex. We first construct the nerve complex $N(C)$ from a
small category $C$. The nerve complex has a structure analogous to
simplicial complexes. Thus we may define its homology.

In order to understand the definition of homology of small categories,
let us first recall the definition of simplicial complexes and their
homology. 

\begin{definition}
 Let $K$ be a simplicial complex with vertex set $V$. 
 For each nonnegative integer $n$, the free Abelian group generated by
 the $n$-dimensional simplices of $K$ is denoted by $C_{n}(K;\Z)$. More
 generally, for a commutative ring $k$, we may form a free $k$-module
 instead of a free Abelian group to obtain $C_{n}(K;k)$.

 In order to make the collection $C_{*}(K;k)=\{C_{n}(K;k)\}_{n\ge 0}$ into
 a chain complex, we assume that the vertex set $V$ is totally ordered.
 When a simplex $\sigma$ has vertices $x_{0},\ldots,x_{n}$ with
 $x_{0}<\cdots < x_{n}$, we denote $\sigma=[x_{0},\ldots,x_{n}]$.
 Now the \emph{$n$-th boundary homomorphism}
 $\partial_{n}: C_{n}(K;k)\to C_{n-1}(K;k)$ is defined by
 \begin{equation}
  \partial_{n}([x_{0},\ldots,x_{n}]) = [x_{1},\ldots, x_{n}] +
   \sum_{i=1}^{n-1} (-1)^{i} [x_{0},\ldots, x_{i-1},x_{i+1},\ldots,x_{n}]
   + (-1)^{n}[x_{0},\ldots,x_{n-1}]. 
 \label{equation:boundary_for_simplicial_complex}
 \end{equation}
 These maps make $C_{*}(K;k)$ into a chain complex, i.e.
 $\partial_{n}\circ\partial_{n+1}=0$ for all $n$.

 The \emph{$n$-th homology group of $K$ with coefficients in $k$} is
 defined by 
 \[
  H_{n}(K;k) = \Ker (\partial_{n}: C_{n}(K;k)\to C_{n-1}(K;k))/\Ima
 (\partial_{n+1}: C_{n+1}(K;k)\to C_{n}(K;k)).
 \]
\end{definition}

When $Q$ is a simple quiver, any element of $N_{n}(Q)$ can be
represented by a sequence of vertices $(v_{0},\ldots,v_{n})$, as we have
observed in Remark \ref{remark:paths_in_simple_quiver}. 
An obvious idea is to form free $k$-modules generated by the sets
$N_{n}(Q)$ and define boundary homomorphisms by a formula similar to
(\ref{equation:boundary_for_simplicial_complex}). 
Unfortunately, $(x_{i-1},x_{i+1})$ may not be an arrow in $Q$, even
if both $(x_{i-1},x_{i})$ and $(x_{i},x_{i+1})$ are arrows in $Q$. The
boundary homomorphism $\partial_{n}$ cannot be defined.

For a small category, however, we may always compose morphisms to get a
new morphism. Thus we may define a chain complex.
In order to simplify the description, we restrict ourselves to the case
of acyclic categories.

\begin{definition}
 A small category $C$ is called \emph{acyclic} if
 \begin{enumerate}
  \item for distinct objects $x,y$, either $C(x,y)$ or $C(y,x)$ is
	empty, and 
  \item for any object $x$, the only morphism from $x$ to $x$ is the
	identity. 
 \end{enumerate}
\end{definition}

\begin{definition}
 An $n$-chain $(u_{n},\ldots,u_{1})$ in $C$ is called
 \emph{nondegenerate} if none of $u_{i}$'s is an identity
 morphism. 
 For $n\ge 1$, the set of nondegenerate $n$-chains in $C$ is denoted by 
 $\barN_{n}(C)$. We also define $\barN_{0}(C)=N_{0}(C)$.
 
 The submodule of $C_{n}(C;k)$ generated by $\barN_{n}(C)$ is denoted by
 $\barC_{n}(C;k)$.  
\end{definition}

\begin{example}
 When a quiver $Q$ does not contain a loop or an oriented cycle, $F(Q)$
 is an acyclic category.
\end{example}

\begin{definition}
 \label{def:homology_of_acyclic_category}
 Let $C$ be a small acyclic category and $k$ a commutative ring.
 The collection $\barC_{*}(C;k)=\{\barC_{n}(C;k)\}_{n\ge 0}$ can be made
 into a chain complex by defining the boundary homomorphisms as follows.
 When $n=1$
 \[
 \bar{\partial}_{1}(u) = t(u) - s(u).
 \]
 For $n\ge 2$, 
 \[
  \bar{\partial}_{n}(u_{n},\ldots,u_{1}) 
   =  (u_{n},\ldots, u_{2}) 
   + \sum_{i=1}^{n-1}(-1)^{i} (u_{n},\ldots,u_{i+1}\circ u_{i},
   \ldots, u_{1})
   + (-1)^{n} (u_{n-1},\ldots, u_{1}).
 \]
 Note that $\bar{\partial}_{n}=0$ for $n\le 0$ by definition.

 The \emph{$n$-th homology of $C$ with coefficients in $k$} is defined
 by 
 \[
  H_{n}(C;k) = \Ker \bar{\partial}_{n}/\Ima \bar{\partial}_{n+1}.
 \]
\end{definition}

\begin{remark}
 The homology groups can be defined for arbitrary small categories. See
 Appendix \ref{math} for details.
\end{remark}

%% file: homology_of_weighted_category.tex
\subsection{Homology of Weighted Quivers and Categories}
\label{homology_of_weighted_category}

Now suppose that our category $C$ is equipped with a weight function $w$
with values in a monoid $W$. We would like to put this information into
the homology of $C$. This can be done when $W$ \emph{acts} on a $k$-module
$M$ from the left, meaning that, for $x\in W$ and $m\in M$, an element
$xm\in M$ is given in such a way that 
\begin{enumerate}
 \item $x(m+m')=xm + xm'$ for $x\in W$ and $m,m'\in M$,
 \item $x(\alpha m)=\alpha xm$ for $x\in W$, $\alpha\in k$, and 
       $m\in M$,
 \item $x(x'm)=(xx')m$ for $x,x'\in W$ and $m\in M$, and
 \item $1m=m$, where $1$ is the unit of $W$.
\end{enumerate}
In other words, $M$ is a \emph{representation} of $W$.

With this information, we modify the definition of homology as follows.

\begin{definition}
 \label{def:homology_of_weighted_acyclic_category}
 Let $C$ be a small acyclic category with a weight function $w:C_{1}\to W$
 and $M$ a representation of $W$.
 For each nonnegative integer $n$, define a $k$-module
 \[
  \barC_{n}(C,w;M) = \barC_{n}(C;k)\otimes M,
 \]
 where the tensor product is taken over $k$.
 The boundary homomorphisms are given as follows.
 When $n=1$
 \[
 \bar{\partial}^{M}_{1}(u\otimes m) = t(u)\otimes w(u)m - 
 s(u)\otimes m.
 \]
 For $n\ge 2$, 
 \begin{multline*}
  \bar{\partial}^{M}_{n}(u_{n},\ldots,u_{1}) 
   = (u_{n},\ldots, u_{2})\otimes w(u_{1})m 
   + \sum_{i=1}^{n-1} (-1)^{i}
 (u_{n},\ldots,u_{i+1}\circ u_{i}, 
   \ldots, u_{1})\otimes m \\
   + (-1)^{n} (u_{n-1},\ldots, u_{1})\otimes m,
 \end{multline*}

 It is elementary to verify that these maps define a chain complex
 $\barC_{*}(C,w;M)$. 
 The \emph{$n$-th homology of $C$ with coefficients in $M$} is defined
 by 
 \[
 H_{n}(C,w;M) = \Ker \bar{\partial}^{M}_{n}/\Ima
 \bar{\partial}^{M}_{n+1}. 
 \]

 When $\Gamma=(Q,w)$ is a weighted quiver, we have a canonical extension
 to a weighted small category $(F(Q),\tilde{w})$ by Lemma
 \ref{canonical_extension} and Example
 \ref{example:from_weighted_quiver_to_weighted_category}.  
 We denote
 \[
  H_{n}(\Gamma;M)=H_{n}(Q,w;M) = H_{n}(F(Q),\tilde{w};M).
 \]
 This is called the \emph{homology of $(Q,w)$ with coefficients in $M$}. 
\end{definition}

This homology group can be regarded as a special case of a
construction, known as the \emph{homology
of a small category with coefficients in a functor}.
A precise meaning is recorded in Appendix \ref{math}.

In general, it is not easy to compute the homology of a small category. 
Fortunately, for categories of the form $F(Q)$, a very small chain
complex for computing the homology is known, which gives us the
following description of $H_{*}(Q,w;M)$.

\begin{theorem}
 \label{first_homology}
 Let $\Gamma=(Q,w)$ be a finite acyclic weighted quiver with
 weights in a monoid $W$ and $M$ be a representation of $W$.
 Define a map
 \[
  \varphi : \bigoplus_{u\in Q_{1}} k\{u\}\otimes M \rarrow{}
 \bigoplus_{x\in Q_{0}} k\{x\}\otimes M
 \]
 by
 \[
 \varphi(u\otimes m) = t(u)\otimes m - s(u)\otimes \left(w(u)\cdot
 m\right). 
 \]
 Then 
 $H_{n}(Q,w;M)=0$ for $n\ge 2$ and 
 \[
  H_{1}(Q,w:M) = \Ker\varphi.
 \]
\end{theorem}

We need to prepare the language of homological algebra to prove this
theorem. A proof is given in Appendix
\ref{functor_homology_derived_functor}. 

We conclude this section by making sample computations of homology of
weighted networks.

\begin{example}
 \label{example:triangle_quiver}
 Let $\Gamma=(Q,w)$ be a simple weighted quiver with three
 vertices 
 $x_{1},x_{2},x_{3}$ shown in 
 Figure \ref{fig:triquiver}, where $w_{1}=w(x_1,x_0)$,
 $w_{2}=w(x_2,x_1)$, and $w_{3}=w(x_2,x_0)$.

 \begin{figure}[htbp]	
  \centering
  \includegraphics[width=.3\textwidth]{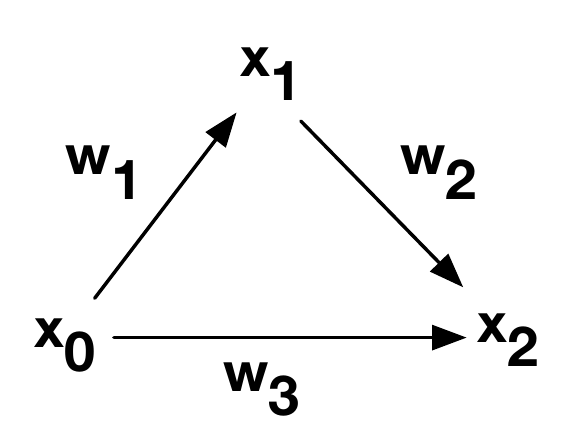}\hfill
  \caption{The weighted quiver $\Gamma$}
  \label{fig:triquiver}
 \end{figure}

 There are two routes from $x_{0}$ to $x_{2}$ in this
 network; the direct route $x_{0}\to x_{2}$ costs $w_{3}$ while the
 route $x_{0}\to x_{1}\to x_{2}$ costs
 \[
  w(x_2,x_1,x_0)=w(x_2,x_1)w(x_1,x_0)=w_2w_1.
 \]
 We would like to
 know the costs of these two routes are equal or not. 
 Let us show that this problem can be solved by computing the first
 homology of the weighted category $(F(Q),w)$.
 
 Suppose that $k$ is a field and that $W$ is a submonoid of
 $k^{\times}=k\setminus\{0\}$. Note that the monoid operation on $W$ is
 given by the multiplication of $k$.
 Then $W$ acts on $k$ by the multiplication.
 The module $k$ with this action is denoted by $k(w)$.
 Let us compute
 \[
  H_{1}(\Gamma;k(w))=H_{1}(Q,w;k(w))=H_{1}(F(Q),\tilde{w};k(w))
 \]
 under these conditions. 


 By Theorem \ref{first_homology}, it suffices to determine $\Ker \varphi$.
 The domain of the map $\varphi$ is a vector space with bases $Q_{1}$,
 which consists of three elements $(x_{1},x_{0})$, $(x_{2},x_{1})$, and
 $(x_{2},x_{0})$.
 The range of $\varphi$ has basis $Q_{0}=\{x_{0},x_{1},x_{2}\}$.


 With these bases the map $\varphi$ is given by
 \begin{align*}
  \bar{\partial}_{1}((x_{1},x_{0})\otimes 1) & = x_{1}\otimes w_{1}\cdot 1 -
  x_{0}\otimes 1 \\
  & = (-1) (x_{0}\otimes 1) + w_{1}(x_{1}\otimes 1) + 0(x_{2}\otimes 1) \\
  \bar{\partial}_{1}((x_{2},x_{1})\otimes 1) & =
  x_{2}\otimes w_{2}\cdot 1 - x_{1}\otimes 1 \\
  & = 0(x_{0}\otimes 1) + (-1)(x_{1}\otimes 1) + w_{2}(x_{2}\otimes 1) \\
  \bar{\partial}_{1}((x_{2},x_{0})\otimes 1) & =
  x_{2}\otimes w_{3}\cdot 1 - x_{0}\otimes 1 \\
  & = (-1) (x_{0}\otimes 1) + 0(x_{1}\otimes 1) + w_{3}(x_{2}\otimes 1).
 \end{align*}
 In other words, the map $\varphi$ is given by the following matrix
 \[
 \begin{pmatrix}
   -1 & 0 & -1 \\
   w_{1} & -1 & 0 \\
   0 & w_{2} & w_3 
 \end{pmatrix}
 \]
 and $\Ker\varphi$ can be identified with the solution to the
 linear equation
 \[
 \begin{pmatrix}
   -1 & 0 & -1 \\
   w_{1} & -1 & 0 \\
   0 & w_{2} & w_3 
 \end{pmatrix}
 \begin{pmatrix}
  a \\ b \\ c
 \end{pmatrix} =
 \begin{pmatrix}
  0 \\ 0 \\ 0
 \end{pmatrix}.
 \]
 

 The determinant of this matrix is
  \[
 \det \begin{pmatrix}
   -1 & 0 & -1 \\
   w_{1} & -1 & 0 \\
   0 & w_{2} & w_3 
 \end{pmatrix}
 = w_3 - w_1w_2. 
 \]
 Thus 
 \[
 \dim \Ker\varphi =
 \begin{cases}
  1, & w_{2}w_{1}=w_{3} \\ 
  0, & w_{2}w_{1}\neq w_{3}.
 \end{cases}
 \]
 When $w_{3}=w_{1}w_2$, a basis for $\Ker\varphi$ can be taken to be the
 vector
 $\begin{pmatrix} 1 \\ w_{1} \\ -1 \end{pmatrix}$.


 Thus the first homology is given by
 \[
 H_{1}(\Gamma;k(w)) = \Ker\varphi
 \cong  
 \begin{cases}
  k \left\langle
  \begin{pmatrix} 1 \\ w_{1} \\ -1 \end{pmatrix} \right\rangle, &
  w_{2}w_{1}=w_{3} \\ 
  0, & w_{2}w_{1}\neq w_{3},
 \end{cases}
 \]
 which means that we can distinguish two cases by looking at the first
 homology. 
\end{example}

\begin{example}
 Consider the weighted quiver $\Sigma=(S,w)$ in Figure
 \ref{fig:square_quiver}. 

 \begin{center}
  \begin{figure}
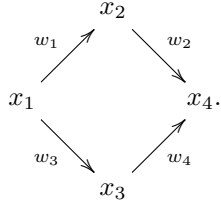

   \[
   \begin{diagram}
    \node{} \node{x_{2}} \arrow{se,t}{w_{2}} \node{} \\
    \node{x_{1}} \arrow{ne,t}{w_{1}} \arrow{se,b}{w_{3}} \node{}
    \node{x_{4}.} \\ 
    \node{} \node{x_{3}} \arrow{ne,b}{w_{4}}
   \end{diagram}
  \]
  \caption{The square quiver $\Sigma$.}
  \label{fig:square_quiver}
  \end{figure}
 \end{center}

 The map $\varphi$ is given by
 \begin{align*}
  \varphi((x_{2},x_{1})\otimes 1) & = x_{2}\otimes w_{1}\cdot 1 -
  x_{1}\otimes 1 \\
  & = (-1) (x_{1}\otimes 1) + w_{1}(x_{2}\otimes 1) + 0(x_{3}\otimes 1)
  + 0(x_{4}\otimes 1) \\
  \bar{\partial}_{1}((x_{4},x_{2})\otimes 1) & =
  x_{4}\otimes w_{2}\cdot 1 - x_{2}\otimes 1 \\
  & = 0(x_{1}\otimes 1) + (-1)(x_{2}\otimes 1) +
  0(x_{3}\otimes 1) + w_{2}(x_{4}\otimes 1) \\
  \bar{\partial}_{1}((x_{3},x_{1})\otimes 1) & =
  x_{3}\otimes w_{3}\cdot 1 - x_{1}\otimes 1 \\
  & = (-1) (x_{1}\otimes 1) + 0(x_{2}\otimes 1) +
  w_{3}(x_{3}\otimes 1) + 0(x_{4}\otimes 1) \\
  \bar{\partial}_{1}((x_{4},x_{3})\otimes 1) & =
  x_{4}\otimes w_{4}\cdot 1 - x_{3}\otimes 1 \\
  & = 0(x_{1}\otimes 1) + 0(x_{2}\otimes 1) + (-1)(x_{3}\otimes 1) +
  w_{4}(x_{4}\otimes 1).
 \end{align*}
 The matrix representation is
 \[
 D= \begin{pmatrix}
   -1    & 0     & -1    & 0     \\
   w_{1} & -1    & 0     & 0     \\
   0     & 0     & w_{3} & -1    \\
   0     & w_{2} & 0     & w_{4}
  \end{pmatrix}.
 \]
 This matrix can be made into the following matrix by row
 transformations.
 \[
 \widetilde{D}= \begin{pmatrix}
   -1 & 0  & -1                    & 0 \\
   0  & -1 & -w_{1}                & 0 \\
   0  & 0  & w_{3}                 & -1 \\
   0  & 0  & w_{4}w_{3}-w_{2}w_{1} & 0
  \end{pmatrix}.
 \]
 The rank of this matrix is
 \[
 \rank D= \rank \widetilde{D} =
 \begin{cases}
  3, &  \text{ if } w_{4}w_{3}=w_{2}w_{1}, \\
  4, &  \text{ if } w_{4}w_{3}\neq w_{2}w_{1}.
 \end{cases} 
 \]
 Thus we obtain
 \[
 \dim H_{1}(\Sigma;k(w)) = \dim \Ker\bar{\partial}_{1} = 4-\rank D = 
 \begin{cases}
  1, & \text{ if } w_{4}w_{3}=w_{2}w_{1}, \\
  0, & \text{ if } w_{4}w_{3}\neq w_{2}w_{1}.
 \end{cases}
 \]

 Again we may tell if $w_{4}w_{3}=w_{2}w_{1}$ or not by computing the
 first homology.
\end{example}

%% file: applications.tex
\section{Applications}
\label{applications}

In this section, we illustrate the practical applicability of our
weighted quiver homology and its corresponding feature vectors to two
well-known tasks on real-world multi-graphs in machine learning and
other graph / network analysis research literature. Namely, we focus on:
(i) Creating node embeddings for weighted directed graphs and (ii)
detecting communities in weighted directed graphs.  

\subsection{Node Embeddings of Weighted Directed Graphs}
\label{subsec:node}

\begin{figure}[tbp]
	\centering
	\subfloat[$H=4$]{{\includegraphics[width=6.8cm]{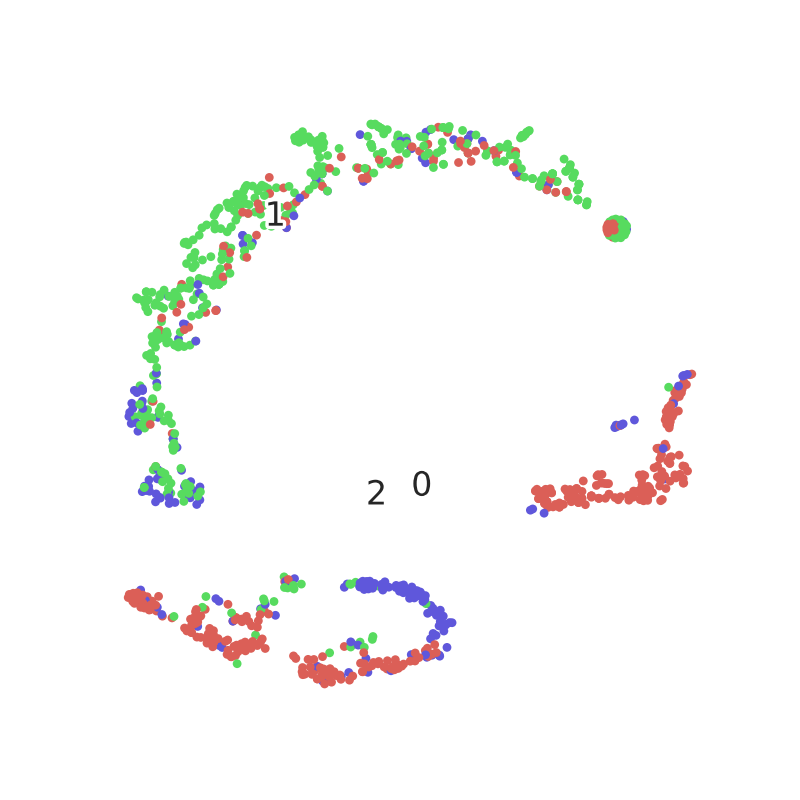} }}%
	\qquad
	\subfloat[$H=5$]{{\includegraphics[width=6.8cm]{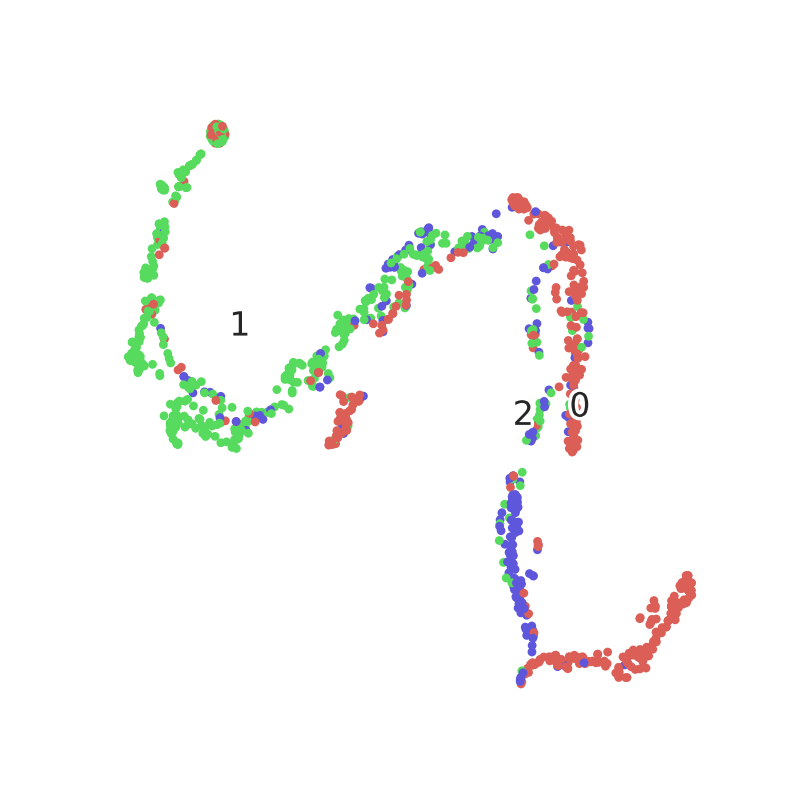} }}%
	\qquad
	\subfloat[$H=6$]{{\includegraphics[width=6.8cm]{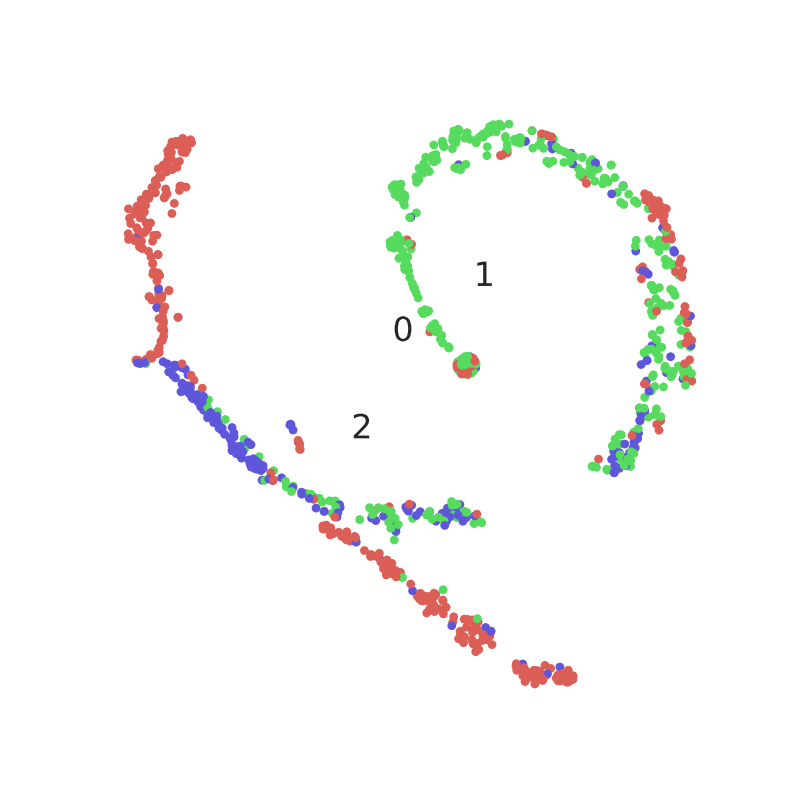} }}%
	\caption{t-SNE plots of the node embeddings as we vary $H$ in Algorithm~\ref{algo:quiv}. With increasing $H$, we notice a better separation of nodes pertaining to different labels / categories. Higher the separation achieved, better the quality of the node embedding.  } 
	\label{fig:node_emb}%
\end{figure}

Given the ubiquitous prevalence of graphs, their analysis in areas like
\emph{machine learning} (ML) plays a fundamental role. In order to apply
existing ML methods to graphs (e.g., to predict new interactions or
discover latent relations between objects represented as nodes /
vertices), one learns a representation of the graph that is amenable to
be used in ML algorithms. 

However, graphs are inherently unordered, irregular, and combinatorial
in nature made up of \emph{nodes / vertices} and \emph{edges / links}
between nodes, while most ML methods (e.g. neural networks) favor
continuous vector representations. To get around the difficulties in
using discrete graph representations in ML, graph embedding methods
learn a continuous vector space for the graph, assigning each node
(and/or edge) in the graph to a specific position in a vector space. We
refer the reader to a survey on node embeddings~\cite{1808.02590} for more
details. 

\begin{description}
 \item[Task:] More formally, given a weighted directed graph
	    $G=(V,E,W)$, where $V$ and $E$ denote the set of nodes and
	    directed edges (arrows) connecting them. $W$ is the set of
	    \emph{edge weights} corresponding to each directed edge
	    $e = (u,v) \in E$. The graph $G$ can be represented by a
	    \emph{weighted adjacency matrix}
	    $A \in \mathds{R}^{|V| \times |V|}$, where the $u,v$-th
	    element in $A$, i.e., 
	    $A_{u,v}$ has a value which corresponds to the edge weight
	    in $W$ of the directed edge $(u,v)$. 
	    In general, node embedding methods try to minimize an objective
	    \[
	    \min_Y L(f(A), g(Y))
	    \]
	    where $Y \in \mathds{R}^{|V| \times d}$, for $d \ll |V|$ is
	    a $d$-dimensional node embedding matrix; 
	    $f: \mathds{R}^{|V| \times |V|} \rarrow{} \mathds{R}^{|V|
	    \times |V|}$ is a transformation of the weighted adjacency
	    matrix; $g: \mathds{R}^{|V| \times d} \rarrow{}
	    \mathds{R}^{|V| \times |V|} $ is a pairwise edge function;
	    and $L:  \mathds{R}^{|V| \times |V|} \rarrow{}
	    \mathds{R}^{|V| \times |V|}$ is a loss function. 

 \item[Dataset:] For our empirical evaluation, we used the popular
	    \emph{Cora
	    dataset}~\cite{Sen-Namata-Bilgic-Getoor-Gallagher-EliassiRad2008}.    
	    The \emph{Cora dataset} is a research citation network
	    (directed) comprising of $2708$ scientific publications
	    classified into one of \emph{seven categories}. The citation
	    network consists of $5429$ links. Each publication (vertex)
	    in the dataset is described by a $0/1$-valued word vector
	    indicating the absence/presence of the corresponding word
	    from the dictionary. The dictionary consists of $1433$
	    unique words. Thus, each vertex has a corresponding binary
	    vector of length $1433$.

 \item[Experimental Setup:] For our experiment, we only focused on a
	    subset of the categories, i.e., three categories, namely
	    \emph{Genetic\_Algorithms} (Label $0$),
	    \emph{Probabilistic\_Methods} (Label $1$), and
	    \emph{Reinforcement\_Learning} (Label $2$).  
	    We computed an edge weight for each edge as the
	    \emph{Jaccard distance} between the vectors associated with
	    the start and terminal vertices of the edge.  

 \item[Results:] In Figure~\ref{fig:node_emb}, we notice that as we
	    increase the number of iterations $H$ in our method, we get
	    a larger dimensional feature vector which starts to achieve
	    better separation of topics / labels among the nodes in the
	    citation network. Therefore, nodes that represent a given
	    topic cluster together and also move away from topics that
	    are different. We see this separation improve as we vary $H$
	    from $4$ to $6$. In order to visualize these $H$-dimensional
	    vectors representing the nodes in the Cora graph, we used
	    $t$\emph{-Distributed Stochastic Neighbor Embedding}
	    (t-SNE)~\cite{vanderMaaten-Hinton2008}, which is a technique
	    for dimensionality reduction that is particularly well
	    suited for the visualization of high-dimensional datasets. 
\end{description}

\subsection{Community Detection in Weighted Graphs}
\label{subsec:comm}

\begin{figure}[tbp]
	\centering
	\subfloat[Original graph]{{\includegraphics[width=14cm]{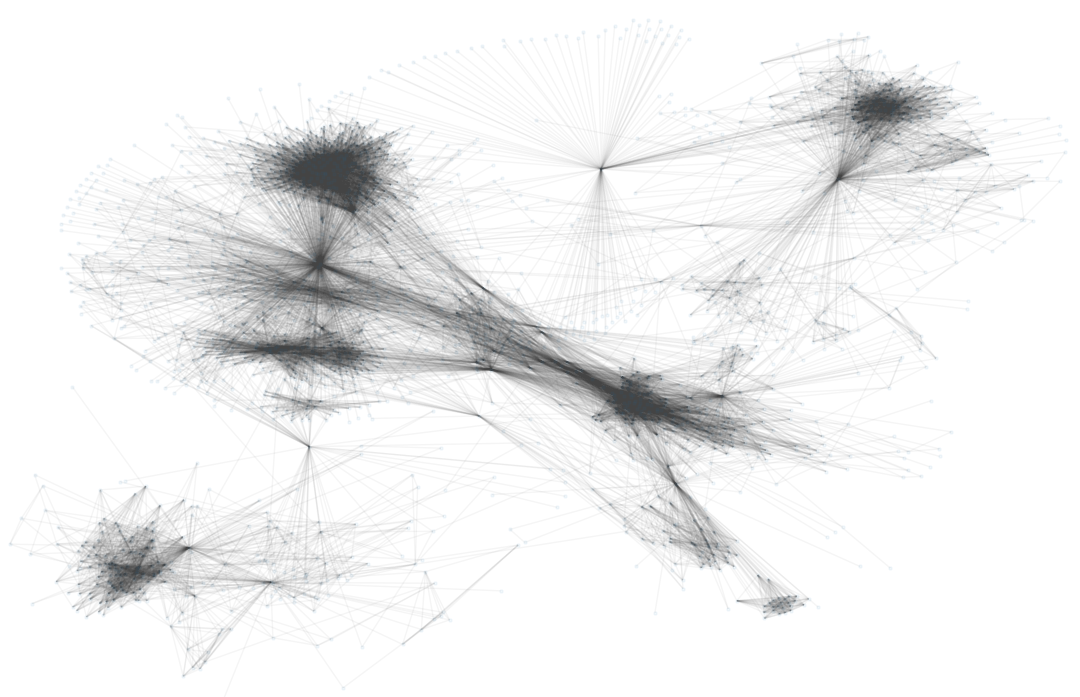} }}%
	\qquad
	\subfloat[Communities marked in original graph]{{\includegraphics[width=14cm]{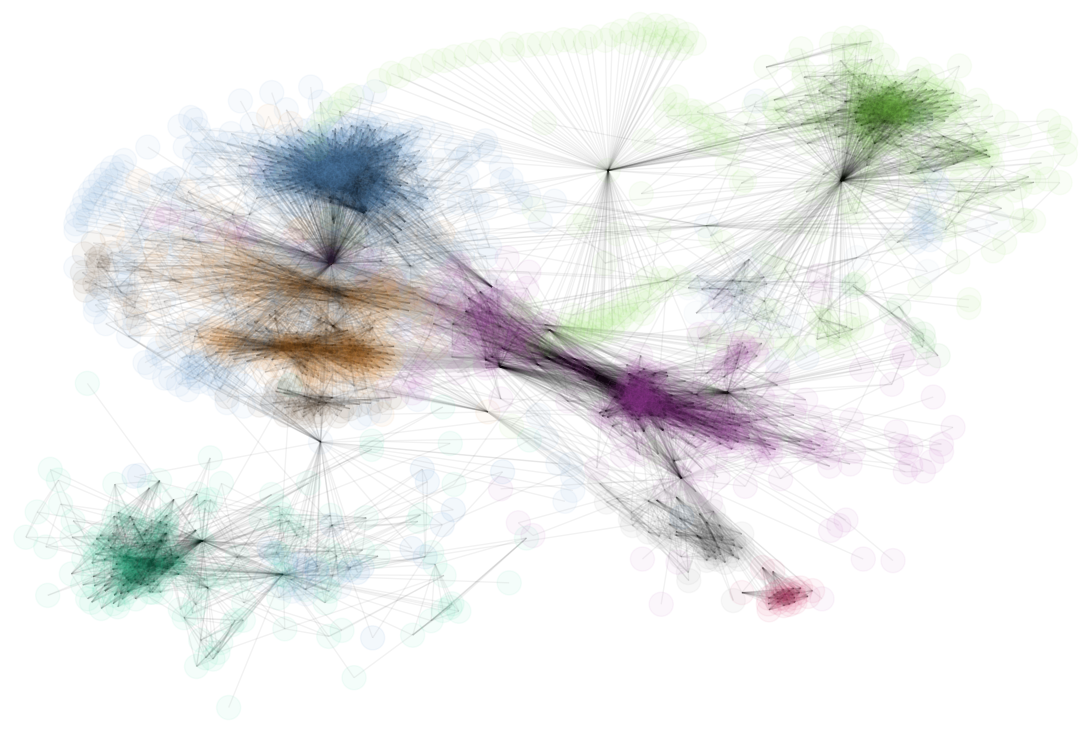} }}%
	\caption{Seven distinct communities (marked with different colored nodes) detected using weighted quiver homology features in Facebook graph with $2094$ nodes and $20$K edges. }%
	\label{fig:comm}%
\end{figure}

Complex systems can be represented in terms of graphs, where the
elements composing the complex system are described as nodes / vertices
and their interactions as edges / links.  
At a global level, the nature of these interactions is far from trivial
and very complex in nature.  
At a mesoscopic (intermediate) scale, it is possible to identify a group
of nodes that are densely connected among themselves, but sparsely
connected to the rest of the graph. Such heavily interconnected group of
vertices are often characterized as \emph{communities} and occur in a
wide variety of networked systems. For example, such communities can be
considered as independent portions of a graph, playing a \emph{similar
role}, like the tissues or the organs in the human body. Community
detection finds applications in a wide and diverse set of areas such as
biology, sociology, and computer science, to name a few, where systems
are often represented as graphs. This problem is extremely hard and has
not yet been solved satisfactorily, despite the huge effort of a large
interdisciplinary community of scientists working on it over the past
few years. This task gets even harder when having to identify such
communities in weighted directed graphs. We refer the reader to a survey
on community detection~\cite{Fortunato2010} for more details. 

\begin{description}
 \item[Task:] Given a graph $G=(V,E)$, we denote the degree of a node
	    $u$ by $\delta_u$. If we consider a subset of nodes $V'
	    \subseteq V$ that are densely connected and represent a
	    \emph{community}, to which node $u$ belongs. We denote the
	    sum of degrees of the nodes present in $V'$ by
	    $\delta_u(V)$. Then, this total degree can be split into two
	    contributions 
	    \[
	    \delta_u(V) = \delta_u^{in}(V') + \delta_u^{out}(V')
	    \]
	    where $\delta_u^{in}(V')$ is the number of edges connecting
	    $u$ to other nodes in $V'$ and  
	    $\delta_u^{out}(V')$ is the number of edges connecting $u$
	    to $V \setminus V'$ (i.e., rest of the nodes outside
	    $V'$). The subset $V'$ is a termed a \emph{community in the
	    strong sense}, if  
	    \[
	    \delta_u^{in}(V') > \delta_u^{out}(V') \text{,
	    \hspace{2em}} \forall u \in V' 
	    \]
 \item[Dataset:] We downloaded the Facebook graph
	    dataset\footnote{\url{http://snap.stanford.edu/data/ego-Facebook.html}}
	    from SNAP~\cite{McAuley-Leskovec2012}. 
	    This dataset consists of \emph{circles} (or \emph{friends
	    lists}) from Facebook. Facebook data was collected from survey
	    participants using this Facebook application. We used a
	    smaller subset of the large graph, by taking into account
	    $2094$ vertices and $20$K edges connecting them.

 \item[Experimental Setup:] As this was an undirected graph dataset, we
	    assigned an orientation to each edge $(u,v)$, by setting $u
	    \rarrow{} v$, if $u < v$, and $u \larrow{} v$, if $u >
	    v$. Accordingly, an edge weight was also assigned as
	    $|u-v|$. $H$ was fixed at $3$, in our experiments. We first
	    computed a node embedding as was done in
	    Section~\ref{subsec:node}, ran a DBSCAN density-based
	    clustering, and mapped the clusters back to the original
	    nodes in the graph. 

 \item[Results:] We detected $7$ different communities that are each
	    uniquely colored and  depicted in
	    Figure~\ref{fig:comm}(b). It can be visually observed that
	    our method does a fairly good job of detecting communities
	    in the strong sense in the Facebook graph. 
\end{description}

%% file: math.tex
\section{Mathematics for Homology of Small Categories}
\label{math}

In this appendix, we collect precise mathematical definitions and
statements for those who have enough mathematical backgroud. Here we
assume that the reader is familiar with basic category theory and
algebraic topology, including simplicial homotopy theory.

\input{functor_homology_definition}
\input{functor_homology_derived_functor}

%% file: functor_homology_definition.tex
\subsection{Homology of Small Categories with Coefficients in Functors}
\label{functor_homology_definition}

Recall that we have introduced the set $N_{n}(Q)$ of $n$-chains in a
quiver $Q$. By regarding a small category $C$ as a quiver, we have a
collection $\{N_{n}(C)\}_{n\ge 0}$ of sets.
When $C$ is a category, this collection has a structure of simplicial
set. 

\begin{lemma}
 \label{def:nerve}
 For a small category $C$, the collection $N(C)=\{N_{n}(C)\}_{n\ge 0}$
 can be made into a simplicial set by the following operators.
 The face operators $d_i : N_n(C) \rightarrow N_{n-1}(C)$ are given as
 follows.
 When $n=1$, $d_{0}(u)=t(u)$ and $d_{1}(u)=s(u)$.
 When $n\ge 2$,  
 \begin{align*}
 (u_{n},\ldots,u_{1}) & = x_0 \rarrow{u_1} \dots \rarrow{u_i} x_i
  \rarrow{u_{i+1}} x_{i+1} 
  \dots \rarrow{u_n} x_n \\ 
  &\xmapsto{d_i} \left \{
  \begin{aligned}
   & x_1 \rarrow{u_2} x_2 \dots \rarrow{u_{n}} x_n, && \text{if}\ i=0 \\
   & x_0 \rarrow{u_1} \dots \rarrow{} x_i \rarrow{u_{i} \circ u_{i+1}}
   x_{i+2} \rarrow{} \dots \rarrow{u_{n}} x_n, 
   && \text{if}\ 
   0 < i < n \\ 
   & x_0 \rarrow{u_{1}} \dots \rarrow{u_{n-1}} x_{n-1}, && \text{if}\ i=n
  \end{aligned} \right.
 \end{align*}
 The degeneracy operators $s_{i}: N_{n}(C) \to N_{n+1}(C)$ are defined
 by 
 \[
  s_{i}(u_n,\ldots,u_{1}) = (u_{n},\ldots, u_{i+1}, 1_{x_{i}} ,
 u_{i},\ldots, u_{1}).
 \]
 This simplicial set is called the \emph{nerve} of $C$.
\end{lemma}

There is a standard way to generate a chain complex from a simplicial
set. 

\begin{definition}
 Let $k$ be a commutative ring.
 For a simplicial set $X$, the free $k$-module generated by $X_{n}$
 is denoted by $C_{n}(X;k)$. Define
 \[
  \partial_{n}: C_{n}(X;k) \rarrow{} C_{n-1}(X;k)
 \]
 by
 \[
  \partial_n = \sum_{i=0}^{n} (-1)^i d_i.
 \]
 The collection $C(X;k)=\{C_{n}(X;k),\partial_{n}\}$ forms a
 chain complex over $k$. The homology of this chain complex is denoted by
 \[
  H_{n}(X;k) = H_{n}(C(X;k))=\Ker\partial_{n}/\Ima\partial_{n+1}
 \]
 and is called the \emph{homology group of $X$ with coefficients in $k$}.

 When $X=N(C)$ for a small category $C$, the homology of $N(C)$ with
 coefficients in $k$ is denoted by $H_{n}(C;k)$.
\end{definition}

When $C$ is equipped with a functor $G:C\to \lMod{k}$,
we may modify the definition of the nerve and homology as follows.

\begin{definition}
 Let $G:C\to \lMod{k}$ be a functor on a small category $C$.
 Define
 \[
  B_{n}(C;G) = \bigoplus_{x\in C_{0}} k\langle
 s^{-1}(x)\rangle\otimes G(x),
 \]
 where $s: N_{n}(C)\to C_0$ is the map defined in Definition
 \ref{def:quiver}. 

 This collection of $k$-modules $B(C;G)=\{B_{n}(C;G)\}_{n}$ can be made
 into a simplicial $k$-module as follows.
 When $n=1$, the face operators are given by
 \begin{align*}
  d_{0}(u\otimes m) & =t(u)\otimes G(u)(m) \\
 d_{1}(u) & =s(u)\otimes m.
 \end{align*} 
 When $n\ge 2$, the face operators are given by
 \[
  d_{i}((u_{n},\ldots,u_{1})\otimes m) 
  = \left \{
  \begin{aligned}
   & (u_{n},\ldots,u_{2})\otimes G(u_{1})(m), && \text{if}\ i=0 \\
   & (u_{n},\ldots,u_{i+1}\circ u_{i},\ldots,u_{1})\otimes m,
   && \text{if}\ 
   0 < i < n \\ 
   & (u_{n-1},\ldots,u_{1})\otimes m, && \text{if}\ i=n
  \end{aligned} \right.
 \]
 The degeneracy operators $s_{i}: B_{n}(C;G) \to B_{n+1}(C;G)$ are defined
 by 
 \[
  s_{i}((u_n,\ldots,u_{1})\otimes m) = (u_{n},\ldots, u_{i+1}, 1_{x_{i}} ,
 u_{i},\ldots, u_{1})\otimes m.
 \]

 The face operators can be assembled in the usual way to define a
 boundary operator
 \[
  \partial_{n} = \sum_{i=0}^{n} (-1)^{i}d_{i} : B_{n}(C;G) \rarrow{}
 B_{n-1}(C;G). 
 \]
 The homology of this chain complex is denoted by $H_{n}(C;G)$ and is
 called the homology of $C$ with coefficients in $G$.
\end{definition}

\begin{example}
 \label{example:representation_as_functor}
 Let $w:C\to W$ be a weight function and $M$ a representation of $W$. 
 When $W$ is regarded as a category with a single object $*$, the left
 action of $W$ on $M$ can be regarded as a covariant functor
 $\mu_{M}:W\to M$, which assigns $M$ to the unique object $*$ in $W$.
 Then the composition 
 \[
  F_{w} = \mu_{M}\circ w: C \rarrow{} \lMod{k}
 \]
 is a functor given by 
 $F_{w}(x) = M$ on objects and
 \[
  F_{w}(u)(m) = w(u)\cdot m
 \]
 for $u\in C_{1}$ and $m\in M$.

 The homology of $C$ with coefficients in $F_{w}$ is essentially the
 homology defined in Definition
 \ref{def:homology_of_weighted_acyclic_category}.  
\end{example}

%% file: functor_homology_derived_functor.tex
\subsection{Homology of Small Categories as a Derived Functor}
\label{functor_homology_derived_functor}

The aim of this section is to prove Theorem \ref{first_homology}.
We first need a description of the homology of small category as a derived
functor. 
In the rest of this section, we free use the language of homological
algebra. 
We also use the following notation which simplifies descriptions of
constructions related to small categories and functors.

\begin{definition}
 Let $C$ be a small category. The $k$-linear category generated by $C$
 is denoted by $kC$ so that $(kC)_{1}$ is the free $k$-module generated
 by $C_{1}$.
 The free $k$-module generated by $C_{0}$ is denoted by $kC_{0}$. We
 regard it as a coalgebra over $k$ under the diagonal on $C_{0}$.
 We regard $(kC)_{1}$ as a right $kC_{0}$-comodule via the source map
 $s$ and a left $kC_{0}$-comodule via the target map $t$.

 For a left $kC_{0}$-comodule $M$, define $kC\Box_{C_{0}}M$ by the
 following equalizer diagram
 \[
  \equalizer{kC\Box_{C_{0}}M}{kC\otimes M}{s\otimes 1}{1\otimes
 \delta_{M}}{kC\otimes kC_{0}\otimes M,} 
 \]
 where $\delta_{M}$ is the comodule structure map of $M$.

 A left $C$-module is a left $kC_{0}$-comdule $M$ equipped with a map
 \[
  \mu_{M} : kC\Box_{C_{0}}M \rarrow{} M
 \]
 satisfying the associativity and unit conditions.
 Right $C$-modules are defined in a similar way by switching $kC$ and $M$.
 The categories of left and right $C$-modules are denoted by $\lMod{C}$
 and $\rMod{C}$, respectively.
\end{definition}

\begin{example}
 \label{functor_as_module}
 Let $G: C\to \lMod{k}$ be a functor. Define a $k$-module $\Gamma(G)$ by
 \[
  \Gamma(G) = \bigoplus_{x\in C_{0}} G(x).
 \]
 We regard $\Gamma(G)$ as a left $C_{0}$-comodule via
 \[
  \delta_{G}(a) = x\otimes a
 \]
 if $a\in G(x)$.
 Then 
 \[
  kC\Box_{C_{0}}\Gamma(G) = \bigoplus_{u\in C_{1}} k\{u\}\otimes G(s(u)).
 \]
 The induced map $G(u): G(s(u))\to G(t(u))$ induces a map
 \[
  k\{u\}\otimes G(s(u)) \rarrow{} G(t(u)),
 \]
 which defines a structure of left $kC$-module on $\Gamma(G)$.

 Similarly, a contravariant functor $G:C^{\op}\to \lMod{k}$ gives rise
 to a right $kC$-module $\Gamma(G)$.
\end{example}

It is well-known that categories $\lMod{C}$ and $\rMod{C}$ are Abelian
categories with enough projectives.
Thus we may define derived functors.
We are interested in the derived functor of the following bifunctor.

\begin{definition}
 Let $C$ be a small category. For a right $C$-module $N$ and a left
 $C$-module $M$, define a $k$-module $N\otimes_{C}M$ by the following
 coequalizer diagram
 \[
  \coequalizer{N\Box_{C_{0}}kC\Box_{C_{0}} M}{\mu_{N}\otimes
 1}{1\otimes\mu_{M}}{N\Box_{C_{0}}M}{N\otimes_{C}M,} 
 \]
 where $\mu_{N}$ and $\mu_{M}$ are module structure maps for $N$ and
 $M$, respectively.
\end{definition}

Let $N=kC_{0}$, regarded as a right $C$-module via the target map
$t:C_{1}\to C_{0}$. 
Then, for a functor $G:C\to\lMod{k}$, we have the following isomorphism
\[
 B_{n}(C;G) \cong kC_{0}\otimes_{C}\underbrace{kC\Box_{C_{0}}\cdots
 \Box{C_{0}}kC}_{n+1}\otimes_{C} \Gamma(G), 
\]
which can be assembled into an isomorphism of chain complexes.
Since the collection
\[
 \left\{\underbrace{kC\Box_{C_{0}}\cdots
 \Box{C_{0}}kC}_{n+1}\otimes_{C} \Gamma(G)\right\}_{n\ge 0}
\]
is a projective resolution of $\Gamma(G)$ in $\lMod{C}$, the general
theory of derived functors implies the following description of homology
of small categories.

\begin{proposition}
 Let $G: C\to \lMod{k}$ be a functor and
 \[
  \cdots \rarrow{} P_{n} \rarrow{d_{n}} P_{n-1} \rarrow{} \cdots
 \rarrow{} P_{1} 
 \rarrow{d_{1}} P_{0} \rarrow{\varepsilon} \Gamma(G) \rarrow{} 0
 \]
 be a projective resolution of $\Gamma(G)$ in $\lMod{C}$.
 Then we have a natural isomorphism
 \[
  H_{n}(kC_{0}\otimes_{C} P_{*}) \cong H_{n}(C;G)
 \]
 for all $n\ge 0$.
\end{proposition}

When $C=F(Q)$ for a finite acyclic quiver $Q$, a very small projective
resolution of left $C$-modules is known.
The following description can be found in a lecture note by
Crawley-Boevey \cite{CrawleyBoevey-quivlecs}.

\begin{proposition}
 \label{quiver_standard_resolution}
 Let $Q$ be a finite acyclic quiver and $G:F(Q)\to \lMod{k}$ be a
 functor. Then the following sequence is exact  
\begin{equation}
 0 \rarrow{} kF(Q)\Box_{Q_{0}} kQ_{1}\Box_{Q_{0}} \Gamma(G)
 \rarrow{f} kF(Q)\Box_{Q_{0}} \Gamma(G)
 \rarrow{g} \Gamma(G) \rarrow{} 0, 
 \label{equation:minimal_resolution}
\end{equation}
 where $kQ_{1}$ is regarded as a right $kQ_{0}$-comodule via the source
 map and a left $kQ_{0}$-comodule via the target map.
 The maps $f$ and $g$ are defined by
 \begin{align*}
  g(a\otimes m) & = G(a)(m), \\
  f(a\otimes u\otimes m) & = a\otimes G(u)(m) - au\otimes m.
 \end{align*}
\end{proposition}

The above sequence is called the \emph{standard resolution} or the
\emph{minimal resolution} of $G$ over the free category $F(Q)$.

Theorem \ref{first_homology} is now a corollary to Proposition
\ref{quiver_standard_resolution}. 

\begin{proof}[Proof of Theorem \ref{first_homology}]
 Since (\ref{equation:minimal_resolution}) is a projective resolution,
 $H_{*}(F(Q);G)$ can be computed by using this resolution for any
 functor $G:F(Q)\to\lMod{k}$. 
 In the case of Theorem \ref{first_homology}, the functor is given by
 \[
  G(x) = M
 \]
 for any $x\in Q_{0}$. Thus
 \[
  \Gamma(G) = \bigoplus_{x\in Q_{0}} k\{x\}\otimes M.
 \]

 For a left $kQ_{0}$-comodule $N$, we have a natural isomorphism
 \[
  kQ_{0}\otimes_{F(Q)}(kF(Q)\Box_{Q_{0}} N) \cong N
 \]
 induced by the target map $t:F(Q)\to Q_{0}$.
 In particular, $H_{*}(Q,w;M)$ is the homology of the complex
 \[
 \cdots \rarrow{} 0 \rarrow{} kQ_{1}\Box_{Q_{0}} \Gamma(G) \rarrow{\bar{f}}
 \bigoplus_{x\in Q_{0}} k\{x\}\otimes M
 \]
 and we have $H_{n}(Q,w;M) = 0$ for $n\ge 2$.
 And the induced map $\bar{f}$ is given by
 \[
  \bar{f}(u\otimes m) = t(u)\otimes (w(u)\cdot m) - s(u)\otimes m.
 \]
 This completes the proof of Theorem \ref{first_homology}.
\end{proof}